\documentclass[pmlr,table]{jmlr}

\usepackage{longtable}
\usepackage{nicefrac}                      
\usepackage{microtype}                     
\usepackage{amsfonts, bm}
\usepackage{mathtools}
\usepackage{breakcites}                    
\usepackage[capitalize]{cleveref}
\usepackage{booktabs}
\usepackage{xcolor}
\usepackage{makecell, multirow}
\usepackage{xurl}

\LinesNumbered

\newcommand{\Var}{\mathrm{Var}}

\newcommand{\A}{\mathcal{A}}

\newcommand{\M}{\mathcal{M}}

\newcommand{\Pa}{\mathcal{P}}

\newcommand{\Mlap}{\mathcal {M}_{\mathrm {Lap}}}
\newcommand{\Mrr}{\mathcal {M}_{\mathrm {RR}}}
\newcommand{\Mgr}{\mathcal {M}_{\mathrm {GRR}}}
\newcommand{\Mr}{\mathcal {M}_{\mathrm {R}}}
\newcommand{\Ml}{\mathcal {M}_{\mathrm {L}}}

\DeclareMathOperator{\E}{\mathbb{E}}

\newcommand{\qed}{\hfill\blacksquare}

\def\eg{{\em e.g.}}

\jmlrvolume{1}
\jmlryear{2022}
\jmlrworkshop{Algorithmic Fairness through the Lens of Causality and Privacy}

\title[``You Can't Fix What You Can't Measure'']{``You Can't Fix What You Can't Measure'': Privately Measuring Demographic Performance Disparities \\in Federated Learning}

 \author{\Name{Marc Juarez\nametag{\thanks{Most of their work was done while at the University of Southern California.}}} \Email{marc.juarez@ed.ac.uk}\\
 \addr University of Edinburgh
 \AND
 \Name{Aleksandra Korolova\nametag{\setcounter{footnote}{0}\footnotemark}} \Email{korolova@princeton.edu}\\
 \addr Princeton University
 }

\begin{document}

\maketitle

\begin{abstract}
  \noindent As in traditional machine learning models, models trained with federated learning may exhibit disparate performance across demographic groups. Model holders must identify these disparities to mitigate undue harm to the groups. However, measuring a model's performance in a group requires access to information about group membership which, for privacy reasons, often has limited availability. We propose novel locally differentially private mechanisms to measure differences in performance across groups while protecting the privacy of group membership. To analyze the effectiveness of the mechanisms, we bound their error in estimating a disparity when optimized for a given privacy budget. Our results show that the error rapidly decreases for realistic numbers of participating clients, demonstrating that, contrary to what prior work suggested, protecting privacy is not necessarily in conflict with identifying performance disparities of federated models.
\end{abstract}


\section{Introduction}\label{sec:intro}
Cross-device federated learning (DFL) has become a popular way to distribute the training of machine learning (ML) models across multiple devices.
Currently, there are several large-scale deployments of DFL in the industry, such as Android GBoard's next-word prediction~\citep{hard2018federated,yang2018applied}, and Siri's speaker identification~\citep{granqvist2020improving}.
A key motivation for these deployments is the aspiration of training powerful models while ensuring privacy and data minimization.

In parallel, ML models have been shown to exhibit disparate performance across groups, often falling short for people from marginalized groups, in the domains of vision, natural language processing, and healthcare~\citep{buolamwini2018gender,sap2019risk,celi2022sources,mehrabi2021survey}; and, more recently, these performance disparities have also been observed in DFL~\citep{yuan2020federated,xu2021privacy,xu2021fedmood}.

Performance disparities may be harmful beyond merely the individual's experience of worse quality of service~\citep{crawford2017trouble}.
A greater false positive rate of Alexa's wake-word detection on a group may lead to over-surveillance of that group, as a false activation may record unrelated speech and send it to the cloud for further processing~\citep{Vitaladevuni2020}.
In DFL applications to the security domain~\citep{hosseini2020federated}, a performance disparity may lead to a lower security level for certain groups.
Overall, DFL has enormous traction in industry and academia and, if it were to become a de-facto data minimization or privacy standard for much of ML, as current trends suggest, an unknown performance disparity dependent on a demographic group could have tremendous negative consequences in many application domains.

The challenge in detecting and mitigating disparate performance of a DFL model is that access to information about the attributes related to group membership is often limited or noisy~\citep{veale2017fairer, bogen2020awareness}. Regulations, such as the GDPR in the EU, mandate that protected attributes, such as gender or race, be collected only under appropriate privacy protections and explicit informed consent. In addition, without adequate privacy protection, volunteers who are members of stigmatized groups are more likely to provide a false group membership, which can add noise to the collected attributes.

We reconcile the seemingly incompatible goals of ensuring group membership privacy and mitigating performance disparity via local differential privacy (LDP). We propose novel LDP mechanisms that allow us to measure performance disparities while protecting the privacy of the attributes that define group membership.
To compare the mechanisms over a range of privacy levels, number of clients, and group sizes, we characterize the measurement error they induce as a function of the privacy budget, and find budget allocations that minimize the error under a privacy constraint.

Our theoretical analysis shows that the mechanisms ensure strong privacy guarantees while the measurement error is relatively low for typical numbers of clients in a DFL setting.
With our tools, the aggregator of the models, or even a regulatory agency could identify cases of performance disparity in applications where such disparity is undesirable or harmful for some of the groups.

\section{Background}\label{sec:bgs}

\paragraph{Differential Privacy (DP)}\label{sec:dp}
We argue that the local model of DP is better suited for cross-device federated learning (DFL), as it is unclear who would play the role of a central curator in the existing deployments of DFL and the large-scale of these deployments can attenuate the privacy-induced error of an LDP mechanism.

\begin{definition}[$\epsilon$-Local Differential Privacy ($\epsilon$-LDP)]\label{def:dp}
A randomized mechanism $\M:D\rightarrow R$ satisfies $\epsilon$-LDP where $\epsilon > 0$ if, and only if, for any pair of inputs $v, v'\in D$ and for all $y\in R$
$$\frac{\Pr[\M(v)=y]}{\Pr[\M(v')=y]}\leq e^{\epsilon},$$
where the probabilities are taken over the randomness of $\M$.
\end{definition}

One of the simplest LDP mechanisms is \emph{Randomized Response} (RR).
For a binary protected attribute, RR returns the true value with probability $a$ and returns the opposite value otherwise.
Generalized RR (GRR) extends RR to a non-binary protected attribute by uniformly distributing the probability of giving a different value~\citep{wang17locally}.

\begin{definition}[The GRR mechanism]
For $x\in \{0, \ldots, d\}$, $d\geq1$, and $a\in [\frac{1}{2}, 1]$, the GRR mechanism, $\Mgr$, is defined by
\[
\Pr[\Mgr(x;d,a)=y] \;:=\; 
 \begin{cases}
  a & \text{if } y=x\\
  \frac{1-a}{d-1} & \text{if } y\neq x
 \end{cases}
\]
\end{definition}
\paragraph{Federated Learning (FL)}
FL allows many clients, each with its own dataset, to train a model on the union of all datasets, without any dataset having to leave its client's device.
The training is distributed over the clients, who train local models on their datasets.
The clients then share the local models' parameters with a central \emph{aggregator}, who averages them to obtain the parameters of the \emph{global model}.

There are different types of FL depending on who the clients are. We focus on cross-device FL (DFL), where clients run on different devices (\eg, smartphones) and the training data usually belongs to the same user.
We focus on DFL as it is becoming increasingly popular in the Big Tech industry~\citep{yang2018applied,granqvist2020improving}.

\section{Problem Statement}\label{sec:prob_statement}
Motivated by the harms of disparate performance of the global model in DFL, the notion of \emph{unfairness} that we consider in the DFL setting is the disparate performance across groups of clients with the groups defined by a demographic attribute, such as sex or race.

Formally, an attribute is a set $\A = \{0,\ldots,d\}$, with $d\geq 1$, that induces a partition of the clients, $\Pa= \{G_0, \ldots, G_d\}$. We consider $K$ clients and denote $(g_k, v_k)$ the values of the attribute ($g_k\in \A$) and the performance of the model for client $k$. The client obtains $v_k$ by evaluating the global model on a fraction of their data. The choice of an appropriate performance metric depends on various factors, such as the learning task and the potential harms in a particular application.

\begin{definition}[\bf{Group mean performance}]\label{eq:mean_perf}
The mean performance of a group $G\in \Pa$ is $m_G := \frac{1}{n}\sum_{i=1}^{n} v_i$, where $n=|G|$.
\end{definition}

To quantify the difference in performance between any two groups, we measure the absolute difference between the mean performances of the global model on the groups.

\begin{definition}[\bf{Performance gap}]\label{eq:perf_gap}
The performance gap between any two $A,B\in\Pa$ is defined by $\Delta m := \left|m_A-m_B\right|$.
\end{definition}

This notion of (un)fairness is in contrast with traditional fairness definitions that measure model performance on individual predictions.
Previous definitions are suitable for scenarios where data points represent people and each single prediction concerns an individual (\eg, credit score prediction).
However, these definitions are not suited for the notion of (un)fairness that we consider. In the typical DFL setting, the global model is distributed to the clients for use on their \emph{own} data and they are not necessarily the subjects of the predictions. Our concern is that the disparate performance of the global model across groups of users can lead to a disparate impact;
the performance gap captures this disparity across groups.

\paragraph{Adversary model.}
We assume that the entity performing the measurements of the performance gap is the FL aggregator, but our mechanisms could also be used by an external entity, such as a regulatory agency or a public interest auditor. As in popular DFL deployments~\citep{mcmahan2017federated}, we assume that the aggregator uses Secure Aggregation~\citep{bonawitz2017practical}.
Therefore, the aggregator cannot infer any information from the FL updates about the users' protected attributes.

Both the group and the model performance value are privacy-sensitive information (the group---because it corresponds to an attribute such as race; the performance---because of the potential correlation with the group).
Thus, the clients must apply an LDP mechanism, $\M$, on their group-value tuples before sending them to the aggregator. We denote the perturbed tuples $(g'_k, v'_k) := \M(g_k, v_k)$.
From a privacy perspective, we assume that the aggregator follows the protocol as intended but may try to learn $g_k$ or $v_k$ from $(g'_k, v'_k)$.

\medskip
In this work, we investigate the question: \emph{how can the aggregator measure the performance gap while protecting the privacy of the clients' $(g_k, v_k)$ tuples with an LDP mechanism?}
To address this question, we design novel LDP mechanisms and study the tradeoffs they impose in terms of their privacy guarantees and the error they induce in a measurement.
\emph{We hypothesize that the number of clients of current DFL deployments is sufficient to allow for low-error and high-privacy measurements with our mechanisms}.

\section{Measuring the Performance Gap}\label{sec:measurements}
Since the performance gap is the absolute difference of two group mean performances, we first tackle the problem of private \emph{group} mean estimation and then show that the privacy guarantees hold when combining them into a performance gap estimation.

A major distinction between \emph{group} mean estimation and \emph{population} mean estimation in the literature~\citep{brown2021covariance,asi2022optimal} is that, in estimating a group mean performance, the performance may be correlated with the group---especially, if there is a gap.
In that case, the adversary would learn group information from observing the performance.

A successful mechanism must protect both group and performance values.
A na\"ive approach is to protect both independently, but that would destroy the necessary information to measure the gap. Thus, our mechanisms are designed to preserve the \emph{overall} aggregate correlation between the group and the performance values, while preventing inference of the group that an individual client belongs to from the perturbed tuples.

All our mechanisms use $\Mgr$ to perturb the group values. The intuition for perturbing the group with GRR is that it provides plausible deniability for group membership. As a result, clients have less incentives to lie, as they can always claim that the mechanism assigned them to a different group.

\setlength{\algomargin}{1.5em}
\begin{algorithm2e}[!ht]
\caption{Pseudocode of the privacy mechanism: $\Mr$.\label{alg:mechanism_rr1}}
\DontPrintSemicolon
\LinesNumbered

\KwIn{The client's group $g\in\{0, 1\}$ and performance value $v\in[-1,1]$. The privacy budgets $\epsilon_1,\epsilon_2\in [0, +\infty)$.}
\KwOut{The perturbed tuple $(g', v')$, $g'\in \{0, 1\}$ and $v'\in \{-1, 1\}$.}

{
\BlankLine
$g' \gets \Mgr(g; d, \frac{e^{\epsilon_1}}{e^{\epsilon_1} + d - 1})\quad$ \tcp{\small Perturb the group.} 
\BlankLine
\If{$g \neq g'$}{
$v\gets 0\quad $  \tcp{\small The performance distribution of the other group is unknown.}}
\BlankLine
Draw $B\sim $ Bernoulli($\frac{1+v}{2}$)\; 
\BlankLine
$v' \gets 2* \Mgr(B; 2, \frac{e^{\epsilon_2}}{1+e^{\epsilon_2}})-1\quad$  \tcp{\small Perturb the value and transform.} 
\BlankLine
\Return $(g', v')$
}
\end{algorithm2e}

We present two mechanisms that differ by how they perturb the performance values:
\paragraph{The $\Mr$ mechanism.} After perturbing the group (line 1 in Algorithm~\ref{alg:mechanism_rr1}), $\Mr$ discretizes the value with Harmony's discretization~\citep{nguyen2016collecting} (lines 3--5) and then applies GRR on the discretized value (line 6). The Harmony discretization allows for unbiased estimates of the expected value from the discretized values.
The mechanism sets to zero the performance value of the clients who flip their group (line 2). This is to ensure that they do not contribute to the other group's mean performance value~\citep{gu2020pckv}.

\paragraph{The $\Ml$ mechanism.} $\Ml$ is identical to $\Mr$ except that instead of applying GRR to perturb the values, it uses the Laplace mechanism. The scale of the noise may vary depending on whether the client's group has been perturbed. Clients whose group flipped do not require as much noise to \emph{hide} in the value distribution of the other group, as those values are also perturbed with Laplace noise. Thus, $\Ml$ exposes a parameter $k$, $0<k\leq 2$, that allows to fine-tune the scale of the Laplace distribution of the noise for the clients whose group was perturbed.
In addition, the value of the clients that switch to another group is set to zero, such that, like in $\Mr$, they do not contribute to that group's mean value.

\medskip
Because the mechanisms are the composition of the GRR and the Laplace mechanisms, they have two privacy budget parameters: $\epsilon_1$ and $\epsilon_2$, the privacy budget to protect the group and the performance values, respectively.

One of our main results is that the mechanisms achieve $\epsilon$-LDP for an overall privacy budget $\epsilon$.

\begin{theorem}\label{theo:rr_ldp}
$\Mr$ is $\epsilon$-LDP with
$\epsilon = \max\{\epsilon_1, \epsilon_2\}$.
\end{theorem}
\begin{proof}\renewcommand{\qed}{}
See \cref{proof:rr_ldp}.
\end{proof}

\begin{theorem}\label{theo:rl_ldp}
The mechanism $\Ml$ is $\epsilon$-LDP with 
\begin{equation}\label{eq:bound_vl}
\epsilon=\max\left\{ \epsilon_2,\
\ln(\frac{2}{k}) + \frac{\epsilon_2}{2} - \epsilon_1,\
\ln(\frac{k}{2}) + \frac{\epsilon_2}{k} + \epsilon_1 
     \right\}
\end{equation}
\end{theorem}
\begin{proof}\renewcommand{\qed}{}
See \cref{proof:rl_ldp}.
\end{proof}

These bounds are tighter than the ones obtained with the basic theorem on sequential composition of DP mechanisms~\citep{mcsherry2009privacy}.
The tightness of the LDP bound is important to provide an uppper bound for the privacy of the mechanism, when comparing it with other mechanisms, as well as quantifying the privacy vs.\ utility tradeoff.

\section{Performance Evaluation}\label{sec:err_analysis}
To measure the privacy-induced error, we follow the LDP literature by treating the measurement as an estimator of $m_G$ under the randomness of the mechanisms.
A key metric of the quality of an estimator is its Mean Squared Error (MSE), as it captures the error due to both the estimator's variance and its bias.
By showing that our estimators are unbiased, we can compare their MSEs by simply comparing their variance.
Further, knowing the estimators' variance allows us to probabilistically bound the distance of a performance gap measurement to its true value as a function of the number of clients, which is informative to assess the feasibility of our mechanisms in a DFL setting.

The estimators that the operator should use to estimate $m_G$ from the perturbed group-performance tuples are as follows.

\begin{definition}[Estimators of $m_G$]
We define the following estimators for the mechanisms:
\begin{equation*}
\tilde{m}^{L}_G:= \frac{1}{a\,n}\sum_{j=1}^{n'} v'_j,\quad\quad\text{and}\quad\quad
\tilde{m}^{R}_G:= \frac{1}{a(2b-1)n}\sum_{j=1}^{n'} v'_j,
\end{equation*}
where $a=\frac{e^{\epsilon_1}}{e^{\epsilon_1}+d-1}$ and $b=\frac{e^{\epsilon_2}}{1+e^{\epsilon_2}}$, $n=|G|$, and $n'$ is the number of clients in group $G$ after the mechanism's perturbation.
\end{definition}

We have proven that the estimators are unbiased, and have obtained closed-form expressions of their variance (see \cref{proof:variances}), hence their MSE.

\begin{proposition}\label{prop:unbiasedness}
The estimators of the mechanisms are unbiased: $\E[\hat{m}^{L}_G]=\E[\hat{m}^{R}_G]=m_G$. 
\end{proposition}

\begin{proof}
\renewcommand{\qed}{}
See \cref{proof:unbiasedness}
\end{proof}

\cref{theo:variances} shows that the MSE of $\Delta \tilde{m}^*$ is the sum of the MSEs of the group mean value estimates.

\begin{theorem}\label{theo:variances}
$\Delta \tilde{m}^{*}$ is an unbiased estimator of $\Delta m$ for two groups $G,\bar{G}\in\Pa$ with MSE:
$$\text{MSE}\left[\Delta \tilde{m}^{*}\right] =  \text{MSE}[\tilde{m}^{*}_G]+\text{MSE}[\tilde{m}^{*}_{\bar{G}}].$$
\end{theorem}
\begin{proof}
See \cref{proof:theo_variances}.
\end{proof}

The intuition behind \cref{theo:variances} is that even though $\tilde{m}^*_G$ and $\tilde{m}^*_{\bar{G}}$ are not independent, the errors are uncorrelated, and thus they add up.
Therefore, we can obtain a closed-form expression of the MSE of $\Delta m$ by adding the closed-form expressions of the group MSEs. 


\paragraph{MSE for a fixed privacy budget}\label{sec:comparison}
As shown by \cref{theo:rr_ldp} and \cref{theo:rl_ldp}, the privacy budgets $\epsilon_1$ and $\epsilon_2$ have a different impact on the LDP bound of the mechanism. To draw a fair comparison between the mechanisms, we need to find the parameters that minimize the error on utility for a fixed overall privacy budget $\epsilon$.
With the closed-form expression of the MSE of $\Delta \tilde{m}^*$, we can compare the estimators for specific values of $\epsilon_1$ and $\epsilon_2$.
It is unclear a priori how to divide a fixed privacy budget into the mechanism's group and performance components to maximize utility.
Our approach is to find an allocation that minimizes the MSE of the estimators, for the total privacy budget of the mechanism ($\epsilon$).

For $\Mr$, the optimal allocation is $\epsilon_1=\epsilon_2=\epsilon$; for $\Ml$ with $k=2$, it is $\epsilon_2=\epsilon$, and $\epsilon_1=\frac{\epsilon}{2}$. In \cref{app:details}, we find the optimal allocation for $\Ml$ with $k$ as a parameter.

\begin{figure}
\centering
\begin{minipage}{.4\textwidth}
\centering
    \includegraphics[scale=0.99]{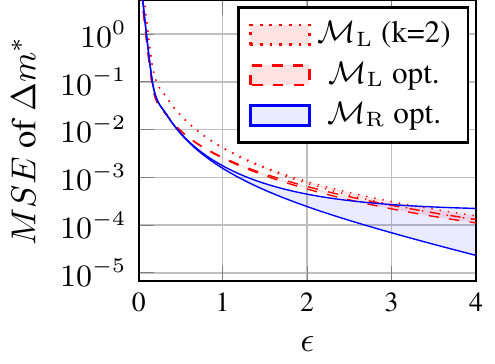}
    \caption{Upper and lower bounds of estimator MSE for different privacy budgets $\epsilon$. We have set $n_G=n_{\bar{G}}=10^{4}$.}
    \label{fig:sum_MSE}
\end{minipage}%
\hfill
\begin{minipage}{.56\textwidth}
  \centering
    \includegraphics[scale=0.93]{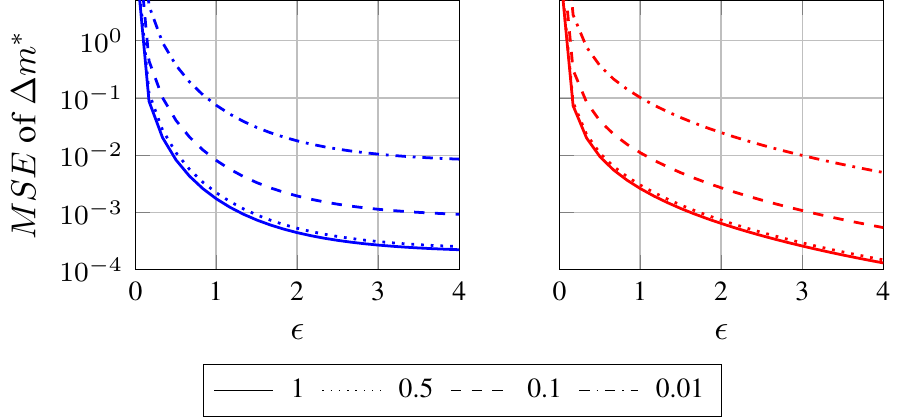}
    \caption{Upper bound of the MSE of $\Mr$ opt.\ (left) and $\Ml$ with $k=2$ (right) for $K=2\cdot 10^{4}$ and different group ratios $n_G/n_{\bar{G}}$.}
    \label{fig:MSE_different_ratio}
\end{minipage}
\end{figure}

\paragraph{Comparison of the mechanisms.}
In \cref{fig:sum_MSE}, we plot the MSE of the performance gap estimator for two groups $G$ and $\bar{G}$ of the same size.
$\Ml$ opt.\  and $\Mr$ opt.\ are the mechanisms with optimal parameters. Since the specific set of $v_k$'s has an impact on the MSE, in the graph we show the lower and upper bounds for each mechanism, which enclose a (colored) region of MSEs.

We observe that $\Ml$ opt.\ achieves lower MSE than $\Ml$ with $k=2$ for the range of $\epsilon$ that we consider.
$\Mr$ opt.'s MSE is lower than $\Ml$ opt.\ only when $0.31 \lessapprox \epsilon \lessapprox 2.6$ but, for $\epsilon \gtrapprox 2.6$,  the upper bound of $\Mr$ opt.'s MSE is larger than the upper bound of $\Ml$ opt.'s and the gap between the two rapidly widens for larger privacy budgets.
As a consequence, there is no overall best mechanism and the operator should select the one that best suits their budget.

\paragraph{Unbalanced Groups}\label{sec:balanced}
We also study the impact that an unbalance of group sizes has on the MSE of the estimators.
\cref{fig:MSE_different_ratio} depicts the upper bounds of the MSEs, $\Mr$ opt.\ (left) and $\Ml$ with $k=2$\ (right), for two groups with different size ratios.
We observe that the mechanisms can cope with relatively small groups, but the MSE rapidly grows for minorities that account for less than 1\% of the clients.
Consequently, the mechanisms would incur in high MSE with protected attributes that include small minorities, but would maintain low MSEs for protected attributes like gender and race.

We have implicitly assumed that the entity performing the measurements knows $n$. \cref{fig:MSE_different_ratio} also quantifies the difference of the mechanisms' MSE when the operator incorrectly estimates the group fractions.
For example, when the groups are balanced and the estimate of the group fractions has a relative error of up to 50\%, the difference in the MSE values (i.e., the difference between the bold and dotted lines) is insignificant. This means that the assumption of knowing $n$ can be relaxed in practice to a large extent.

\paragraph{Precision Relative to the Number of Clients}\label{sec:conf_int}
Our mechanisms allow for high-precision measurements with realistic numbers of clients in the DFL setting. Using \cref{theo:variances}, we obtain a Chebyshev concentration bound and numerically find the privacy budget $\epsilon$ such that $|\Delta \tilde{m}_G^*-\Delta m| < \alpha$, for a small $\alpha>0$, with high probability.

\begin{table}[t]
    \centering
    \caption{Minimum privacy budget ($\epsilon$) required to bound the error by $\alpha$, given $K$ clients, with 0.99 probability. Highlighted are the $\epsilon$'s that are considered reasonable.
    }
    \label{tab:epsilon_tradeoff}
\resizebox{0.8\columnwidth}{!}{
        \begin{tabular}{ccccccc} 
        \toprule 
        \multicolumn{1}{c}{} & \multicolumn{3}{c}{$\boldsymbol{\Mr}$ opt.} & \multicolumn{3}{c}{$\boldsymbol{\Ml}$ opt.} \\
        \cmidrule(lr){2-4} \cmidrule(lr){5-7}
        $\boldsymbol{K}$ & $\alpha=10^{-1}$ & $\alpha=10^{-2}$ & $\alpha=10^{-3}$ & $\alpha=10^{-1}$ & $\alpha=10^{-2}$ & $\alpha=10^{-3}$\\ 
        \cmidrule(lr){1-7}
         $10^5$ &           \cellcolor{green!25}1.86 & - & - &           \cellcolor{green!25}2.56 & 17.89 & 178.89 \\
         $10^6$ &           \cellcolor{green!25}0.63 & - & - &           \cellcolor{green!25}0.71 &  6.32 &  56.57 \\
         $10^7$ &           \cellcolor{green!25}0.23 &  \cellcolor{green!25}1.86 & - &           \cellcolor{green!25}0.21 &  \cellcolor{green!25}2.56 &  17.89 \\
         $10^8$ &           \cellcolor{green!25}0.08 &  \cellcolor{green!25}0.63 & - &           \cellcolor{green!25}0.07 &  \cellcolor{green!25}0.71 &   6.32 \\
         $10^9$ &           \cellcolor{green!25}0.02 &  \cellcolor{green!25}0.23 &  \cellcolor{green!25}1.86 &           \cellcolor{green!25}0.02 &  \cellcolor{green!25}0.21 &   \cellcolor{green!25}2.56 \\
        \bottomrule 
        \end{tabular}
    }
\end{table}

\cref{tab:epsilon_tradeoff} shows the minimum privacy budget required to ensure that the error is at most $\alpha$ for various values of $\alpha$ and numbers of clients, with probability $0.99$.
If the operator can afford a higher privacy budget, the bound would still hold but if their privacy budget is lower, the mechanism does not guarantee the bounds.

Due to $\Mr$'s discretization step, the maximum estimator variance (attained when all the performance values are zero) tends to a constant (indirectly proportional to group size). Thus, there is no budget that allows to achieve $\alpha$ for those cells marked with ``-''.

These results show that the required budgets to achieve an error of less than one percentage point and less of one tenth of a percentage point are reasonable for $K\geq 10^7$ and $K \geq 10^9$ clients, respectively.
Even though these may look like large numbers of clients, current DFL deployments have this many, and even more, clients.
For example, in 2018 Apple reported a total of half a billion active Siri clients~\citep{applesiri2018} and, in the same year, Gboard surpassed 1 billion installs~\citep{googlegboard2018}.

We have also evaluated these bounds on a real-world dataset of performances of a simulated FL deployment (\cref{sec:emp_evaluation}). Our results show that the bounds not only hold but that they are overly conservative: in practice, the operators of the mechanisms would be able to ensure the same privacy level by spending less privacy budget. The Chebyshev bounds are known to be loose because they do not make assumptions about the underlying distributions; we leave finding tighter bounds for future work.

\section{Related Work}
In the ML literature, \citet{veale2017fairer} first noted the legal, institutional, and commercial deterrents against collecting demographic data.
To address the lack of demographic data, they envisioned privacy-preserving protocols that rely on a third-party to detect and mitigate discrimination.

Researchers materialized these protocols with DP and Secure Multi-Party Computation (SMPC)~\citep{jagielski2019differentially, kilbertus2018blind,metatech}.
\citet{jagielski2019differentially} proposed DP versions of existing post-\ and in-processing techniques to train classifiers that satisfy the Equalized Odds constraint.
In contrast, our work defines the notion of performance gap as it is more suitable for the DFL setting.
In addition, as most of related work, their focus is on mitigating the disparity assuming that it has occurred, while we focus on the measurement of the disparity rather than its mitigation.

In contrast to SMPC, our mechanisms have low computational and bandwidth costs, and are robust to client dropouts.
Moreover, SMPC and DP provide different privacy guarantees; in particular, SMPC does not limit the information about individual group membership that the aggregator can infer from the measurements.

Another difference between our work and the prior works is the involvement of the model holder towards the goal of identifying disparities. Depending on how DFL is implemented, our approach may allow the clients and the mechanism operator to measure the performance gap without the aggregator's collaboration.

Prior work on LDP mechanisms to protect sensitive attributes is too extensive to be covered in detail. Recent work has made progress on designing mechanisms for private mean estimation on the theoretical~\citep{nguyen2016collecting,asi2022optimal} and practical fronts~\citep{gu2020pckv,ye2019privkv}. However, this literature does not consider the perturbation of performance values for a performance gap measurement, and therefore, is focused on slightly different privacy vs.\ utility tradeoffs than we are.

\section{Discussion and Conclusion}\label{sec:conclusion}
With FL gaining traction in industry and academia, there is a growing concern that models trained with it will exhibit disparate performance across demographic groups, leading to harms ranging from a mere inconvenience to disparate impact, such as increased surveillance and lower online security for some of the groups.
We propose considering the performance gap between demographic groups as a notion of (un)fairness in the DFL setting, and argue that the ability to measure it is crucial towards addressing such harms.
However, especially under the privacy aspirations of federated learning, lack of demographic data hinders the applicability of existing techniques to measure performance disparities in DFL models.
This poses an obstacle to mitigating the harms; as Roy \citet{austin2021race}, Facebook's VP of Civil Rights, puts it: ``we can't address what we can't measure.''

To address the legal, societal, and individual concerns related to the privacy of demographic data, we propose locally differentially private mechanisms that estimate the performance gap while protecting the privacy of the group membership and potentially correlated data such as model performance.
Our theoretical and experimental results show that the mechanisms ensure strong privacy guarantees while performing relatively precise performance gap measurements when relying on realistic numbers of clients in the DFL setting and reasonable privacy parameters.
Our insight is that the large scale of existing DFL deployments offers a unique opportunity to  measure and expose the potential disparities while guaranteeing strong privacy to the participants.

\acks{We thank the anonymous reviewers for their helpful comments. We are also grateful to Shuang Wu, Nedelina Teneva, and Basileal Imana for their valuable input during the development of this work.
This work has been supported in part by USC + Amazon Center on Secure \& Trusted ML, and NSF Awards \#1943584, \#1916153, and \#1956435.}

\bibliography{main.bib}

\appendix
\pagebreak
\section{Proofs}\label{sec:appendix_proofs}
\subsection{Proof of Theorem~\ref{theo:rr_ldp}}\label{proof:rr_ldp}
\begin{proof}

We denote $a=\frac{e^{\epsilon_1}}{e^{\epsilon_1} + d - 1}$ and $b=\frac{e^{\epsilon_2}}{1+e^{\epsilon_2}}$.
Let $x_0=(g_0,\ v_0)$ and $x_1=(g_1,\  v_1)$ be two different inputs and $y=(g',\ v')$ be an output of the mechanism. From the mechanism's definition, we have that for an arbitrary input $x=(g,\ v)$,
\[
\Pr[y\mid x]\;=\;
 \begin{cases}
  \frac{a(1+(2b-1)v'v)}{2} & \text{if } g'=g\\[0.7em]
  \frac{1-a}{2(d-1)} & \text{if } g'\neq g \\
 \end{cases}
\]

We prove it for $d=2$ as that is what we use in most of our evaluation, and leave the case $d>2$ for future work.

Since $v\in [-1, 1]$ and $v'\in\{-1, 1\}$, an upper bound of  $\Pr[y\mid x]$ when $g'=g$ is
\begin{equation}\label{eq:upper}
  \Pr[y\mid x]\leq ab
\end{equation}
and a lower bound is
\begin{equation}\label{eq:lower}
   \Pr[y\mid x]\geq a(1-b)
\end{equation}

Now, we bound $\Pr[y\mid x_0]/\Pr[y\mid x_1]$, where $x_0$ and $x_1$ differ in either group or value.
If they have the same group but may (or may not) differ in value, we consider two cases: $g'=g$ and $g'\neq g$ (where $g=g_0=g_1$).

\paragraph{Case 1: $g'=g$.}

Using the upper and lower bounds, we obtain:
\begin{equation}\label{eq:value1}
    \frac{\Pr[y\mid x_0]}{\Pr[y\mid x_1]}\leq \frac{ab}{a(1-b)}=e^{\epsilon_2}
\end{equation}

\paragraph{Case 2: $g'\neq g$.}
Using the probability of $\Pr[y\mid x_1]$ when $g'\neq g$:
\begin{equation}\label{eq:value2}
    \frac{\Pr[y\mid x_0]}{\Pr[y\mid x_1]}=1\leq e^{\epsilon_2}, \text{ as } \epsilon_2\in [0, +\infty)
\end{equation}

This shows that if the inputs have the same group, the differential privacy guarantee boils down to the guarantee of the value-perturbing GRR mechanism.

If $x_0$ and $x_1$ differ in group, we again break down the analysis into two cases: $g'=g_0\neq g_1$ and $g'=g_1\neq g_0$.

\paragraph{Case 1: $g'=g_0\neq g_1$.}
Using the upper bound and taking $e_2=0$ as the minimum value for the denominator, we obtain:
\begin{equation}\label{eq:g=g'}
    \frac{\Pr[y\mid x_0]}{\Pr[y\mid x_1]}\leq \frac{2ab}{1-a}=\frac{2e^{\epsilon_2+\epsilon_1}}{1+e^{\epsilon_2}}\leq e^{\epsilon_1}
\end{equation}

\paragraph{Case 2: $g'=g_1\neq g_0$.}
Using the lower bound and that $1\leq e^{\epsilon_2}$, we have:
\begin{equation}\label{eq:gnotg'}
    \frac{\Pr[y\mid x_0]}{\Pr[y\mid x_1]}\leq \frac{1-a}{2a(1-b)}=\frac{1+e^{\epsilon_2}}{2e^{\epsilon_1}}\leq \frac{2e^{\epsilon_2}}{2e^{\epsilon_1}} =  e^{\epsilon_2-\epsilon_1}
\end{equation}

\noindent
Combining the equations above, we conclude that $\Mr$ is $\epsilon$-DP with $\epsilon=\max\{\epsilon_1,\ \epsilon_2,\ \epsilon_2 - \epsilon_1\} = \max\{\epsilon_1,\epsilon_2\}$ and, thus, the optimal budget allocation is $\epsilon_1=\epsilon_2=\epsilon$.
\end{proof}

\subsection{Proof of Theorem~\ref{theo:rl_ldp}}\label{proof:rl_ldp}
\begin{proof}
This proof is for $k=2$. Let $x_0=(g_0,\ v_0)$ and $x_1=(g_1,\ v_1)$ be two different inputs and $y=(g',\ v')$ be an output of the mechanism. Because $\Ml$ perturbs the values with Laplacian noise, we have that for an arbitrary input $x=(g,\ v)$,
\[
\Pr[y\mid x]\;=\;\begin{cases}
  \frac{e^{\epsilon_1}}{e^{\epsilon_1}+d -1} \;f_{\mathrm{Lap}(0, \frac{2}{\epsilon_2})}(v'- v) & \text{if } g'=g\\[0.7em]
  \frac{1}{e^{\epsilon_1}+d -1} \;f_{\mathrm{Lap}(0, \frac{2}{\epsilon_2})}(v') & \text{if } g'\neq g \\
 \end{cases}
\]

This is because when the mechanism preserves the group, $v'= v + Y$ where $Y\sim \mathrm{Lap}(0, \frac{2}{\epsilon_2})$, hence the probability of the new value is the probability of sampling $v'- v$ from the Laplace distribution with zero mean and scale of $\frac{2}{\epsilon_2}$. When the group is flipped, the mechanism sets $v$ to zero therefore in that case it is the probability of sampling $v'$ from $\mathrm{Lap}(0, \frac{2}{\epsilon_2})$.

As in the proof of \cref{theo:rr_ldp}, we follow a case-based reasoning. If $x_0$ and $x_1$ have the same group but differ in value, we consider two cases: $g'=g$ and $g'\neq g$.

\paragraph{Case 1: $g'= g$.}
\begin{equation}
    \frac{\Pr[y\mid x_0]}{\Pr[y\mid x_1]} = \frac{ f_{\mathrm{Lap}(0, \frac{2}{\epsilon_2})}(v'- v_0)}{ f_{\mathrm{Lap}(0, \frac{2}{\epsilon_2})}(v'-v_1) } = e^{\epsilon_2(\frac{|v'-v_1|}{2}-\frac{|v'-v_0|}{2})} \leq e^{\epsilon_2}
\end{equation}

\paragraph{Case 2: $g'\neq g$.}
\begin{equation}
    \frac{\Pr[y\mid x_0]}{\Pr[y\mid x_1]} = \frac{ f_{\mathrm{Lap}(0, \frac{2}{\epsilon_2})}(v')}{ f_{\mathrm{Lap}(0, \frac{2}{\epsilon_2})}(v') } = 1
\end{equation}

If $x_0$ and $x_1$ differ in group, we again consider two cases: $g'=g_0\neq g_1$ and $g'=g_1\neq g_0$.

\paragraph{Case 1: $g'= g_0 \neq g_1$.}

$$
\frac{\Pr[y\mid x_0]}{\Pr[y\mid x_1]} = \frac{\frac{e^{\epsilon_1}}{e^{\epsilon_1}+d -1} f_{\mathrm{Lap}(0, \frac{2}{\epsilon_2})}(v'- v_0)}{ \frac{1}{e^{\epsilon_1}+d-1} f_{\mathrm{Lap}(0, \frac{2}{\epsilon_2})}(v') } \\
=  e^{\epsilon_1 + \epsilon_2 (\frac{|v'|}{2}-\frac{|v'-v_0|}{2})} \\
\leq e^{\epsilon_1+\frac{\epsilon_2}{2}}
$$

The last inequality follows from $\frac{|v'|}{2}-\frac{|v'-v_0|}{2} \leq \frac{1}{2}$.

\paragraph{Case 2: $g'=g_1\neq g_0$.}
\begin{equation}
    \frac{\Pr[y\mid x_0]}{\Pr[y\mid x_1]} =   e^{\epsilon_2 (\frac{|v'-v_1|}{2}-\frac{|v'|}{2})-\epsilon_1}\leq e^{\frac{\epsilon_2}{2}-\epsilon_1}
\end{equation}

The last inequality follows from the triangle inequality: $\frac{|v'-v_1|}{2}-\frac{|v'|}{2} \leq \frac{|v_1|}{2} \leq \frac{1}{2}$.

Finally, combining all the inequalities above, we obtain the $\epsilon$ in the bound of the probability ratio

\[
\epsilon = \max\left\{ \epsilon_2,\
\frac{\epsilon_2}{2} - \epsilon_1,\
\frac{\epsilon_2}{2} + \epsilon_1 
     \right\} = \max\left\{ \epsilon_2,
\frac{\epsilon_2}{2} + \epsilon_1 
     \right\}
\]
\end{proof}

Thus, the optimal budget allocation for mechanism $\Ml$ with $k=2$ is $\epsilon_2=\epsilon$ and $\epsilon_1 = \frac{\epsilon}{2}$.

\subsection{Proof of Theorem~\ref{prop:unbiasedness}}\label{proof:unbiasedness}
\begin{proof}
We prove that $\hat{m}_G^L$ is unbiased. The proof for the unbiasedness of $\hat{m}_G^R$ is analogous.

We model the values in $G$ after applying $\Ml$ with the following mutually independent random variables
\begin{align}
    V_i = B_i(v_i+Y_i), &\quad i=1,\ldots,\ n,\\[0.7em]
    \bar{V}_j = \bar{B}_j(0+\bar{Y}_j) = \bar{B}_j\bar{Y}_j, &\quad j=1,\ldots,\ K-n
\end{align}
where $V_i$ and $\bar{V}_j$ are the final, perturbed values in group $G$ that originate from group $G$ and $\bar{G}$, respectively. In our notation, the bar denotes that the random variable relates to group $\bar{G}$, the complement of $G$.
The random variables $B_i\sim \mathrm{Bernoulli}(a)$ and $\bar{B}_j\sim \mathrm{Bernoulli}(1-a)$ model $\Mrr$, and $Y_i\sim \mathrm{Lap}(0, 2/\epsilon_2)$ and $\bar{Y}_j\sim \mathrm{Lap}(0, k/\epsilon_2)$ model $\Mlap$.
Thus, the expected value of the estimator is
\begin{small}
\begin{align}
\E[\hat{m}_G^L] \;=\; & \frac{1}{a\, n}\left(\sum_{i=1}^{n}\E\left[V_i\right]+\sum_{j=1}^{K-n}\E[\bar{V}_j]\right)&& \text{\small{Linearity of $\E$}}\\[0.7em]
\;=\; & \frac{1}{a\, n}\sum_{i=1}^{n}\E[B_i(v_i+Y_i)]&& \text{\small{$\E[\bar{V}_j]=0$}}\\[0.7em]
\;=\; & \frac{1}{a\, n}\sum_{i=1}^{n}\E[B_i](v_i+\E[Y_i])&& \text{\small{Mutual independence}}\\[0.7em]
\;=\; & \frac{1}{a\, n}\sum_{i=1}^{n}\E[B_i]v_i&& \text{\small{$\E[Y_i]=0$}}\\[0.7em]
\;=\; & \frac{a}{a\, n}\sum_{i=1}^{n} v_i&& \text{\small{$\E[B_i]=a$}}\\[0.7em]
\;=\; & m_G    
\end{align}
\end{small}
We used that $\E[\bar{V}_j]=0$ because $\E[\bar{Y}_j]=0$ and that the random variables are mutually independent.
\end{proof}

\subsection{Closed-form expressions of Variance}\label{proof:variances}
Using the probabilistic model defined in \cref{proof:unbiasedness}, we can write the variance of the estimator $\hat{m}_G^L$ as
$$\Var[\hat{m}_G^L] = \frac{1}{a^2n^2}\Var\left[\sum_{i=1}^{n} (v_i + Y_i)B_i + \sum_{j=1}^{K-n} \bar{Y}_j \bar{B}_j\right].$$
Note that the noise terms have positive variance and therefore do not cancel out. 
We can use the fact that the variables are mutually independent to write the variance of the sum as the sum of variances. We will then obtain variances of products and will use the well-known formula for the variance of the product of two independent random variables. Rearranging the terms gives the closed expression of the variance:
\begin{equation}\label{eq:var_vl}
\Var[\hat{m}_G^{L}] = \frac{1}{n}\left(\nu^2 e^{-\epsilon_1} + \left(1+e^{-\epsilon_1}\right)\left(\sigma^2_L + \frac{K-n}{n}{\bar{\sigma}^2_L}e^{-\epsilon_1}\right) \right)
\end{equation}
where $\nu^2 = \frac{1}{n}\sum_{i=1}^{n} v_i^2$, and $\sigma^2_L, {\bar{\sigma}^2_L}$ are the variances of the Laplace noise distributions (functions of $\epsilon_2$), for clients who do not swap and those who do, respectively.
The lower and upper bounds shown in \cref{fig:sum_MSE} are taken using that $0\leq \nu^2\leq 1$.

The closed-form expression of $\hat{m}_G^R$'s variance can be obtained similarly, and is 

\begin{equation}
    \Var[\hat{m}_G^{R}]=\frac{1}{a(2b-1)^2n}\left(1 - a(2b-1)^2\nu^2 + \frac{K-n}{n}\frac{1-a}{a}\right)
\end{equation}

Recall that $a$ and $b$ are functions of the privacy budgets.

\subsection{Proof outline for Theorem~\ref{theo:variances}}\label{proof:theo_variances}
First, we prove the unbiasedness of $\Delta \hat{m}^{*}$. Due to \cref{prop:unbiasedness} and the linearity of expectation, the expected value of $\Delta \hat{m}^*$ is $\Delta m$. Assuming that $G$ is the advantaged group and thus $\hat{m}^*_G \geq \hat{m}^*_{\bar{G}}$, we have that
$\E[|\hat{m}^*_G - \hat{m}^*_{\bar{G}}|]=|m_G - m_{\bar{G}}|$.

To show that the variance of $\Delta \hat{m}^*$ is the sum of the variance of the mean group value estimators, it suffices to show that $Cov(\hat{m}^*_G, \hat{m}^*_{\bar{G}})=0$, which is true if, and only if, $\E[\hat{m}^*_G \hat{m}^*_{\bar{G}}]=m_G\,m_{\bar{G}}$. Calculating the value of that expectation explicitly, we observe that many of its terms have an independent Laplace r.v.\ as a factor and, consequently, these terms are zero.
Finally, we can apply Bienaym\'e's identity to obtain the result of the theorem.

The proof for $\Mr$ is similar, as the expected value of clients with the group perturbed is zero.

\section{Allocating the privacy budget for the \texorpdfstring{$\Ml$}{GRR+Laplace} mechanism}\label{app:details}
In \cref{eq:var_vl}, we see that the variance of the unbiased estimator for $\Ml$ is dominated by $\epsilon_2$.
Therefore, since $\epsilon_1, \epsilon_2$, and $k$ must satisfy \cref{eq:bound_vl}, we minimize the MSE by first setting $\epsilon_2=\epsilon$ and, then, finding the $k$ that maximizes $\epsilon_1$ under the LDP constraint in \cref{eq:bound_vl}.

If we take $\epsilon_2=\epsilon$ in \cref{eq:bound_vl} of \cref{theo:rl_ldp}, we obtain bounds for $\epsilon_1$
\begin{equation}\label{eq:bound_e1}
    \ln(\frac{2}{k}) -\frac{\epsilon}{2}\leq \epsilon_1 \leq \ln(\frac{2}{k})+\frac{\epsilon}{2}\,\lambda(k),
\end{equation}
where $\lambda(k) = 2\left(1-\frac{1}{k}\right)$. Thus, this inequality holds iff $\frac{2}{3}\leq k$.\looseness-1

To find the $k$ that maximizes $\epsilon_1$, we consider two cases: $0 <\epsilon < 2/3$, and $2/3 \leq \epsilon$. 
\begin{figure}[t]
\centering
\resizebox{\columnwidth}{!}{
   \includegraphics[scale=0.04]{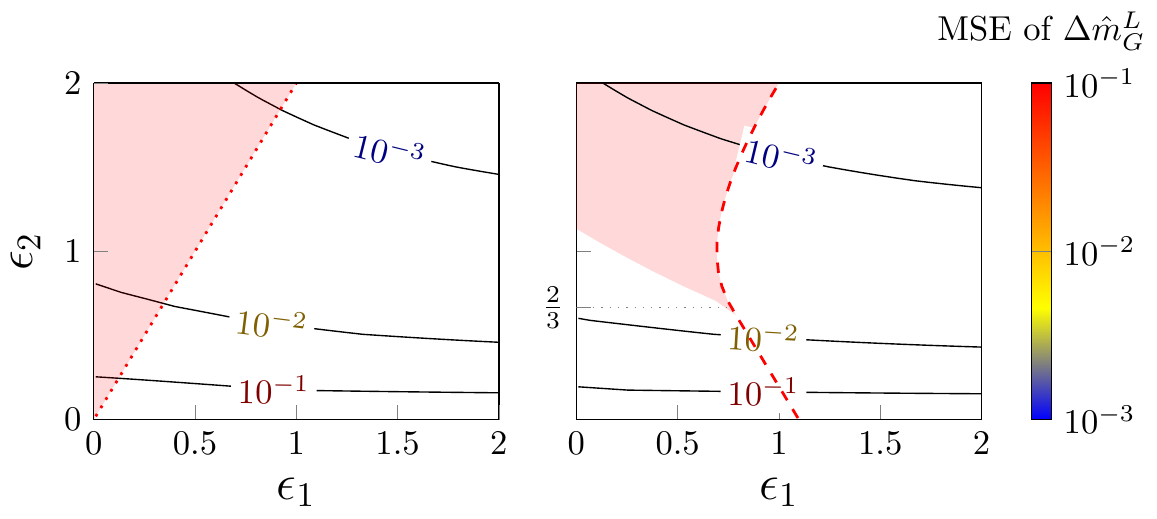}
}
    \caption{Contour plot of the MSE of $\Delta \hat{m}^L$ for $k=2$ (left) and $k=\frac{2}{3}$ (right), as a function of $\epsilon_1$ and $\epsilon_2$. The colored area is the region where the parameters satisfy $\epsilon$-LDP for $\epsilon=\epsilon_2$. The curves represent the optimal allocations when $k=2$ (dotted) and $k=\frac{2}{3}$ (dashed).}
    \label{fig:contour_plot}
\end{figure}

If $2/3 \leq \epsilon$, we write $\epsilon_1$ as the upper bound of $\epsilon$ in \cref{eq:bound_e1}, a function of $k$, and find that $k=\epsilon$ is a maximum for a constant $\epsilon$.
However, for $0<k< 2/3$, \cref{eq:bound_e1} does not hold and hence $k=\epsilon$ would not satisfy $\epsilon$-LDP. When $0<\epsilon < 2/3$, we take $k=2/3$, the minimum $k$ that satisfies $\epsilon$-LDP, as that minimizes the scale of the Laplace noise. In that case, $\epsilon_1$ is equal to the upper and lower bounds in \cref{eq:bound_e1}.

Thus, the maximum $\epsilon_1$ as a function of $\epsilon$ is

\[
\epsilon_1 \;=\;
 \begin{cases}
  \ln(\frac{2}{\epsilon})+\epsilon-1 & \text{if } \frac{2}{3} \leq \epsilon\\[0.7em]
  \ln(3)-\frac{\epsilon}{2} & \text{if } 0 < \epsilon < \frac{2}{3}
 \end{cases}
\]

\cref{fig:contour_plot} shows the allocations of the privacy budgets that satisfy the LDP constraint (colored area).
The dashed and dotted borders of the areas show the allocations that minimize the MSE for a total privacy budget of $\epsilon=\epsilon_2\in(0,2]$ for $k=2$ and $k$=2/3, respectively.
A closer look at the MSE contour lines reveals that the mechanism with $k=2/3$ achieves lower MSE values than for $k=2$ when $\epsilon < 2/3$.

\section{Empirical Validation}\label{app:validation}
We have run experiments to validate the correctness of our expressions of the variance of the estimators.
In the experiments, we initialize two groups with 10 clients each with fixed performance values. Then, we run the mechanisms a number of times to obtain sets of perturbed tuples and calculate the performance gap estimates.
The empirical MSE is the average of the squared differences between these estimates and the true performance gap.
We plot the empirical and theoretical MSE for mechanism $\Mr$ in \cref{fig:theo_vs_emp}. We observe that, as we increase the number of runs, the empirical MSE converges to the theoretical MSE, validating our results.

\begin{figure}[ht]
    \centering
\includegraphics[]{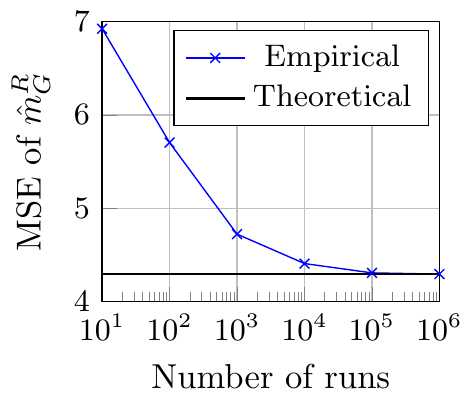}
    \caption{The theoretical upper bound of the MSE of $\hat{m}_G^R$ as derived from \cref{theo:variances}, and its empirical MSE over different runs of $\Mr$, for $n_G=n_{\bar{G}}=10$.}
    \label{fig:theo_vs_emp}
\end{figure}

The source code for reproducing these experiments is publicly available~\citep{Juarez22}.

\section{Empirical Evaluation}\label{sec:emp_evaluation}
We now describe the experiments to evaluate the error of the mechanisms.
Since we are not aware of public datasets with sufficient data to model a real-world deployment of DFL, we  synthesize a dataset by fitting the marginal probability distributions of the protected attribute on a real-world dataset.
Our results show that the error of the mechanisms in the synthetic data is orders of magnitude lower than the Chebyshev bounds obtained in the previous section, indicating that an operator who uses the Chebyshev bounds might be overly conservative in their privacy risk assessment.

\paragraph{Data Generation}
Our data generation model is based on the activity detection dataset collected by~\citet{shinmoto2016effectiveness}.
The dataset comprises the sensor readings for 14 subjects who were instructed to perform a number of scripted daily activities in two different rooms.
The features include the sensor's readings of time, accelerometer position, and radio signal's strength, frequency, and phase.
The labels describe one of these activities: sitting, lying down, or ambulating. We binarized the detection task by relabeling the data to whether or not the subject was lying down.

We define ``sex'' as the protected attribute in the data.
Although the sex of the subject was annotated per each trial---25 male and 62 female---there is no mapping between trials and subjects.
Thus, we assume that each recorded session represents a different FL client, with each client having an average of 864 samples.
We stratify the data ensuring that all clients have the same data distribution between training and test sets (70\% of the samples for training and 30\% for testing).

We simulated the federated learning of a model by training a logistic regression model. We assume that this is the global model trained with the data of all clients.
Since the performance of the model was nearly perfect, resulting in almost all the clients having a zero false positive rate, we have dropped some of the accelerometer features to increase the difficulty of the learning task.
The global model's hold-out average test accuracy for 10 runs is 84.37\%, with a false positive rate (FPR) of 10.69\%, and a true positive rate (TPR) of 82.05\% (all SD values are smaller than 1\%).
Then, we independently test the global model on each client's test set, resulting in two performance values for each client.
We take the TPR and the FPR as performance metrics: the mean TPRs are $89.01$\% and $71.77$\% and the mean FPRs are $15.26$\% and $24.90$\% for males and females, respectively.
We observe a significant performance gap on both metrics: $\Delta$TPR $=17.33\%$ and $\Delta$FPR $=9.63\%$.\looseness-1

\begin{table}[htbp]
    \centering
    \caption{Comparison of the Chebyshev bounds with the empirical mean error for 10 runs of the mechanisms on the synthetic dataset with $K=10^7$ clients. The first column is the privacy budget, followed by the mean error (and standard deviation) of the estimates on the data and the 0.99-probability Chebyshev's bounds ($\boldsymbol{\alpha}$) for each mechanism. }
    \label{tab:error_emp}
    \resizebox{0.65\columnwidth}{!}{
        \begin{tabular}{rcccc}
        \toprule
        \multicolumn{1}{c}{} & \multicolumn{2}{c}{$\boldsymbol{\Mr \text{\textbf{ opt.}}}$} & \multicolumn{2}{c}{$\boldsymbol{\Ml \text{\textbf{ opt.}}} $} \\
\cmidrule(lr){2-3} \cmidrule(lr){4-5}
         $\boldsymbol{\epsilon}$ & $\boldsymbol{|\Delta \hat{m}^R-\Delta m|}$ &  $\boldsymbol{\alpha}$  & $\boldsymbol{|\Delta \hat{m}^L-\Delta m|}$ &   $\boldsymbol{\alpha}$\\ 
        \cmidrule(lr){1-5}
0.01  &      0.1241($\pm$0.1410) &              1.2586 &       0.0504($\pm$0.0337) &              1.0525 \\
0.10  &       0.0082($\pm$0.0059) &              0.1206 &       0.0046($\pm$0.0040) &              0.1060 \\
1.00  &       0.0008($\pm$0.0006) &              0.0094 &       0.0008($\pm$0.0005) &              0.0118 \\
10.00 &       0.0001($\pm$0.0000) &              0.0032 &       0.0001($\pm$0.0000) &              0.0009 \\
        \bottomrule
        \end{tabular}
    }
\end{table}

\paragraph{Regression model implementation}\label{app:implementation}
We implemented the evaluation of the logistic regression model with Python 3.7.6 and sklearn 0.22.1. 

We use Elastic-Net loss (with a 0.99 L1 component) and SAGA as the algorithm to minimize it. To balance the classes, we adjust class weights inversely proportional to class frequency.
To find these hyperparameters we do not optimize for best generalization performance, as we are interested in inducing an disparate performance between the groups.\looseness-1

We evaluated the model selection by 10 runs of hold-out cross-validation (70--30\% as the random training--testing split). We fix the PRNG seed and release the source code included in the supplementary material.

\BlankLine
We published the data and the source code to reproduce these experiments~\citep{Juarez22}.

\paragraph{Error of the DP mechanism}
To generate synthetic data for the global model's performance on new clients, we model the marginal distribution of sex to have the same mean and $\nu^2$ as the observations.
For the purpose of evaluating the error of the mechanisms, the exact distribution that we fit is not important, thus we draw samples with replacement from the set of observations.
This sampling methodology ensures that the relevant statistics are preserved and we generate enough data to represent a realistic DFL deployment.

\cref{tab:error_emp} compares the empirical error with the $0.99$-probability bounds ($\alpha$) obtained with the procedure explained in the previous section, for a range of privacy budgets ($\epsilon$).
The bounds are one order of magnitude larger than the actual error. This means that the budget that the operator would need to allocate to satisfy a certain $\alpha$ for $10^7$ clients is substantially lower than the ones shown in \cref{tab:epsilon_tradeoff}.
As a consequence, following the Chebyshev bounds from the previous section would result in an overly conservative
measurement with respect to the privacy of the users, and operators with small privacy budgets could afford more accurate
measurements without an impact on user privacy.

\end{document}